\documentclass[10pt]{article}

\usepackage{soul}
\usepackage{amsthm}
\usepackage{amsmath}
\usepackage{amssymb}
\usepackage[dvipsnames]{xcolor}
\usepackage{graphicx}
\usepackage{parskip}
\usepackage{fullpage}

\usepackage{natbib}
\definecolor{light-gray}{gray}{0.85}
\definecolor{dark-gray}{gray}{0.40}

\begin{document}
\title{Statistical Queries and Statistical Algorithms: Foundations and 
Applications\footnote{
This survey was first given as a tutorial on statistical queries at the \emph{29th International Conference on Algorithmic Learning Theory} (ALT)~\citeyear{ALT2018}.
Since then, various researchers have asked about the tutorial slides or noted the slides' usefulness 
in helping them to absorb or teach this material.
Hence, that tutorial has been developed into this survey paper in hopes that it will serve as a useful primer
on this subject.  
To give proper attention to all of the authors of the various results, full author lists for the cited papers were provided
en lieu of the customary \emph{et~al.}\ abbreviation.}}

\author{Lev Reyzin\\ 
\normalsize{Department of Mathematics, Statistics, and Computer Science}\\ 
{\normalsize University of Illinois at Chicago}\\ 
\normalsize{\texttt{lreyzin@uic.edu}}}
\date{}

\newcommand{\E}{\mathbf{E}}
\renewcommand{\Pr}{\mathbf{Pr}}

\newcommand{\SQ}{\mathrm{SQ}}
\newcommand{\err}{\mathrm{err}}
\newcommand{\sqdim}{\mathrm{SQ\mbox{-}DIM}}
\newcommand{\ssqdim}{\mathrm{SSQ\mbox{-}DIM}}
\newcommand{\vcdim}{\mathrm{VC\mbox{-}DIM}}

\newcommand{\poly}{\mathrm{poly}}

\newtheorem{theorem}{Theorem}
\newtheorem{corollary}[theorem]{Corollary}
\newtheorem{definition}[theorem]{Definition}
\newtheorem{lemma}[theorem]{Lemma}
\newtheorem{observation}[theorem]{Observation}
\newtheorem{result}[theorem]{Result}
\newtheorem{problem}[theorem]{Problem}
\theoremstyle{plain}

\maketitle

\begin{abstract}
We give a survey of the foundations of statistical queries and their 
many applications to other areas.
We introduce the model, give the main definitions, and we
explore the fundamental theory statistical queries and 
how how it connects to various notions
of learnability.  We also give a detailed summary of some of the
 applications of statistical queries to other areas, including to optimization, to evolvability, and to differential privacy.
\end{abstract}

\section{Introduction} 

Over 20 years ago, \citet*{Kearns98} 
introduced \emph{statistical queries} as a framework for
designing machine learning algorithms 
that are tolerant to noise. 
The statistical query model restricts a learning algorithm to ask certain types of
queries to an oracle that responds with approximately correct answers.
This framework has has proven useful, not only for designing noise-tolerant 
algorithms, but also for its connections to other noise models, for
its ability to capture many of our current techniques, and for its explanatory
power about the hardness of many important problems.

Researchers have also found many connections between statistical queries
and a variety of modern topics, including
to evolvability, differential privacy, and adaptive data analysis. 
Statistical queries are now both an important tool and remain
a foundational topic with many important questions.  The aim of this survey 
is to illustrate these connections and bring researchers to the forefront of our
understanding of this important area.

We begin by formally introducing the model and giving the main definitions (Section~\ref{sec:modeldef}),
we then move to exploring the fundamental theory of learning statistical queries and 
how how it connects to other notions
of learnability (Section~\ref{sec:bounds}).  Finally, we explore many of the
other applications of statistical queries, including to optimization, to evolvability, and to differential privacy (Section~\ref{sec:applications}).

\section{Model, definitions, and basic results}\label{sec:modeldef}

Statistical query learning traces its origins to the \emph{Probably Approximately Correct} (PAC)
learning model of \citet*{Valiant84}.  The PAC model defines the basic supervised learning framework
used in machine learning.  We begin with its definition.\\

\begin{definition}[efficient PAC learning]\label{def:pac}
Let $C$ be a class of boolean functions $c : X \rightarrow \{-1, 1\}$. We say that
$C$ is \underline{efficiently PAC-learnable} if there exists an algorithm $\mathcal{A}$ such that
for every $c \in C$,
any probability distribution $D_X$ over $X$, and any $0 < \epsilon, \delta < 1$, 
$\mathcal{A}$ 
takes a labeled sample $S$ of size $m = \mathrm{poly}(1/\epsilon,1/\delta,n,|c|)$  
from\footnote{$n = |x|$} $D$, outputs a hypothesis $h$ for which
$$\Pr_{S\sim D}[\err_D(h) \le \epsilon] \ge 1-\delta$$
in time polynomial in $m$.
\end{definition}

A useful way to think about the statistical query (SQ) framework is as a restriction on the algorithm $\mathcal{A}$ in the definition above. 
In the SQ model, the learner access to an oracle instead of to a set $S$ of labeled examples.  

The oracle accepts query functions and tolerances, which together are called a statistical query. 
To define the model, we first make this notion presice.\\

\begin{definition}[statistical query]
A \underline{statistical query} is a pair $(q, \tau)$ with
\begin{itemize}
\item[$q$:] a function $q: X \times \{-1,1\} \rightarrow \{-1,1\}$.
\item[$\tau$:] a tolerance parameter $\tau \ge 0$.
\end{itemize}
\end{definition}

\noindent Now we are ready to define the statistical query oracle.\\

\begin{definition}[statistical query oracle]
The \underline{statistical query oracle}, SQ$(q,\tau)$, when given a statistical query, returns any value in the range: $$\left[\E_{x \sim D}[q(x,c(x)] - \tau, \E_{x \sim D}[q(x,c(x)]  + \tau\right].$$
\end{definition}

\noindent Finally, we can give the definition of efficient statistical query learning.\\

\begin{definition}[efficient SQ learning]\label{def:sqlearning}
Let $C$ be a class of boolean functions $c : X \rightarrow \{-1, 1\}$. 
We say that $C$ is \underline{efficiently SQ-learnable} if there exists an algorithm $\mathcal{A}$ such that
for every $c \in C$, any probability distribution $D$, and any $\epsilon > 0$, there is a polynomial $p(\cdot,\cdot,\cdot)$ 
such that \begin{enumerate}
\item $\mathcal{A}$ makes at most $p(1/\epsilon,n,|c|)$ calls to the SQ oracle,
\item the smallest $\tau$ that $\mathcal{A}$ uses satisfies ${1}/{\tau} \le  p(1/\epsilon,n,|c|)$, and 
\item the queries $q$ are evaluable in time 
 $p(1/\epsilon,n,|c|)$,
\end{enumerate}
and $\mathcal{A}$
outputs a hypothesis $h$ satisfying
err$_D(h) \le \epsilon$.
\end{definition}
Note that unlike Definition~\ref{def:pac}, this definition has no failure parameter $\delta$.  That is because in PAC learning, it is possible to get an uninformative sample, whereas
the SQ oracle is restricted to \emph{always} answer queries within a given range.

\subsection{Simulating by algorithms that draw a sample}

It is not hard to see that a statistical query algorithm can be simulated in the PAC model, which makes
SQ a natural restriction of PAC.  In particular one can simulate
an SQ oracle in the PAC model by drawing  $m = O\left(\frac{\log(k/\delta)}{\tau^2}\right)$ 
samples for each of the $k$ statistical queries,
and by the Hoeffding bound, the simulation will fail with probability $< \delta$.  This leads to the 
following observation.\\

\begin{observation}
If a class of functions is efficiently SQ-learnable, then it is efficiently PAC learnable.
\end{observation}

More importantly, learnability with statistical queries is also related to learnability under the classification noise model of \citet*{AngluinL87}.\\

\begin{definition}[classification noise]
A PAC learning algorithm under random \underline{classification noise}
must meet the PAC requirements, but the label of 
each training sample is flipped with independently with probability $\eta$, for $0 \le \eta < 1/2$.  The sample size
and running time must also depend polynomially on $1/(1-2\eta)$.
\end{definition}

This leads us to the following surprising theorem, which shows that any statistical query algorithm
can be converted into a PAC algorithm under classification noise.\\

\begin{theorem}[\citet*{Kearns98}]
If a class of functions is efficiently SQ-learnable, then it is efficiently learnable in the noisy PAC model. 
\end{theorem}
\begin{proof}
For each of $k$ queries, $q(.,.)$, with tolerance $\tau$, let 
Let $P = \E_{x \sim D}[q(x,c(x))]$. We estimate $P$ with $\hat{P}$ as follows.

First, draw a sample set $S$, with 
$|S|= \mathrm{poly}\left({1}/{\tau}, {1}/{1-2\eta}, \log {1}/{\delta}, \log k \right)$
sufficing.
Given $q$, we separate $S$ into two parts\footnote{Note that this does not require knowing the labels of the examples.}: 
\begin{align*}
S_\mathrm{clean} =&\ \{x \in S\ |\ q(x,0) = q(x,1)\} \\
S_\mathrm{noisy}  =&\ \{x \in S\ |\ q(x,0) \neq q(x,1)\}.
\end{align*}

Then, we estimate $q$ on both the parts, with 
\begin{align*}
\hat{P}_\mathrm{clean} =&\ \frac{\sum_{x \in S_\mathrm{clean}} q(x,\ell(x))}{\left| S_\mathrm{clean} \right|}\\
\hat{P}_\mathrm{noisy} =&\ \frac{\sum_{x \in S_\mathrm{noisy}} q(x,\ell(x))}{\left| S_\mathrm{noisy} \right|}.
\end{align*}
Finally, since we know the noise rate $\eta$, we can undo the noise on the noisy part and combine the estimate:
$$\hat{P} = \frac{\hat{P}_\mathrm{noisy} - \eta}{1-2\eta}\left(\frac{|S_\mathrm{noisy}|}{|S|}\right) + 
\hat{P}_\mathrm{clean}\left(\frac{|S_\mathrm{clean}|}{|S|}\right).$$
By the Hoeffding and union bound, we can show that $\hat{P}$ is within $\tau$ of $P$ with probability 
at least $1-\delta$ for all $k$ queries for the $|S|$ as chosen above.
 \end{proof}

Therefore, the SQ framework gives us a way to design algorithms that are also noise-tolerant under some notions
of noise.  In addition, SQ learnability also gives results for learning in the malicious noise model of \citet*{Valiant85},
for example as illustrated in the following Theorem.\\

\begin{theorem}[\citet*{AslamD98a}]
If a class of functions is efficiently SQ-learnable, then it is efficiently PAC learnable under malicious noise
with noise rate $\eta = \tilde{O}(\epsilon)$.
\end{theorem}

\subsection{Variants of SQs}

One natural restriction of statistical queries was defined by \citet*{BshoutyF02},
who modified the oracle to only output the approximate correlation between a query and the target function.
In this \emph{correlational statistical query} (CSQ) model, the oracle is weaker, but the learning criterion is the same as for statistical queries, as in Definition~\ref{def:sqlearning}.

\begin{definition}[correlational statistical query oracle]
Given a function $h = X \rightarrow \{-1,1\}$ and a tolerance parameter $\tau$,
the \underline{correlational statistical query oracle} CSQ$(h,\tau)$ returns a value
in the range  $$[\E_D[h(x)c(x)] - \tau, \E_D[h(x)c(x)] + \tau].$$
\end{definition}

The correlational statistical query oracle above gives \emph{distances} between the hypothesis and a target function.
This is equivalent to the ``Learning by Distances" model of \citet*{Ben-DavidIK95}, who defined their model
independently of~\citet*{Kearns98}.

Another natural way to define statistical queries presented by \citet*{Yang05} is via the \emph{honest statistical query} (HSQ) oracle.  This oracle
samples the distribution and honestly computes approximate answers.\\

\begin{definition}[honest statistical query oracle]
Given function $q: X \times \{-1,1\} \rightarrow \{-1,1\}$ and sample size $m$, the 
\underline{honest statistical query oracle} HSQ$(q,s)$ draws $\left(x_1, \ldots , x_m\right) \sim D^m$ and returns the empirical average
 $$\frac{1}{m}\sum_{i=1}^m q(x_i,c(x_i)).$$
\end{definition}

The definition of honest statistical query learning is again similar to Definition~\ref{def:sqlearning}, but needs some modification
to work with the HSQ oracle.
First, instead of bounding $1/\tau$, the largest sample size $m$ needs to be bounded by a polynomial.  Also, because of the sampling 
procedure, a failure parameter $\delta$ needs to be (re-)introduced, and the learner is required to also be polynomial in ${1}/{\delta}$.

Note that because the CSQ oracle is weaker, any lower bound against SQ algorithms also holds against CSQ algorithms.
On the other hand, the HSQ oracle is arguably stronger and cannot answer adversarially; hence, SQ
algorithms can be easily adapted to give HSQ guarantees.

\section{Bounds for SQ algorithms}\label{sec:bounds} 

We now examine some fundamental theory for statistical query algorithms, beginning with information-theoretic lower bounds that hold
against statistical query algorithms.

\subsection{Lower bounds}

The main tool for proving statistical query lower bounds is called statistical query dimension.  We present it and variants of it in the following section.

\subsubsection{Statistical query dimension}

A quantity called the {statistical query dimension} \citep*{BlumFJKMR94} 
controls the complexity of statistical query learning. \\

\begin{definition}[statistical query dimension]\label{def:sqdim}
For a concept class $C$ and distribution $D$, the {statistical query} {dimension of $C$ with
respect to $D$}, denoted 
$\sqdim_D(C)$, 
is the largest number $d$ such that $C$ contains $d$ functions
$f_1, f_2, \ldots, f_d$ such that for all $i \neq j,$ $\left|\left<f_i,f_j\right>_D\right|  \le 1/d.$
Note: $\left<f_i,f_j\right>_D =  \E_D[f_i \cdot f_j ].$

When we leave out the distribution $D$ as a subscript, we refer to the statistical query dimension with respect to the worst-case
distribution
$$\sqdim(C) = \max_{D \in \mathcal{D}} \left( \sqdim_D(C) \right).$$ 
\end{definition}

This quantity is important due to the following theorem.\\

\begin{theorem}[\citet*{BlumFJKMR94}]\label{thm:sqdimlb}
Let $C$ be a concept class and let $d = \sqdim_D(C)$. Then any SQ learning
algorithm that uses a tolerance parameter lower bounded by $\tau > 0$ must make
at least $(d \tau^2 -1)/2$ queries to learn $C$ with accuracy at least $\tau$.
In particular, when $\tau = 1/d^{1/3}$, this means $(d^{1/3} - 1)/2$ queries are needed.
\end{theorem}

\begin{proof}
The original proof is a bit too technical to present here, so instead we'll see a clever, short 
proof of this lower bound for CSQs given by~\citet*{Szorenyi09}.  This proof gives a weaker result than the statement of the theorem
as proven by \citet*{BlumFJKMR94}.

Assume $f_1, \ldots, f_d$ realize the SQ-DIM. Let $h$ be a query and 
$A = \{ i \in [d]: \left< f_i, h \right> \ge \tau \}$. Then by Cauchy-Schwartz, we have
\begin{equation}\label{eq:szorenyibd}
\left< h, \sum_{i \in A} f_i \right>^2  \le  \left|\left| \sum_{i \in A} f_i \right|\right|^2  
=  \sum_{i,j \in A} \left<f_i,f_j \right> 
 \le  \sum_{i \in A} \left( 1+\frac{|A|-1}{d} \right) 
\end{equation}
therefore $$\left< h, \sum_{i \in A} f_i \right>^2 \le |A| + \frac{|A|^2}{d}.$$  But by definition of $A$, we also have 
$$\left< h, \sum_{i \in A} f_i \right> \ge |A|\tau.$$  By algebra, $|A| \le d/(d\tau^2-1)$, and the same bound holds for $A'$ defined with respect to correlation $\le -\tau$.\\

So, no matter what $h$ is asked of the oracle, an answer of $0$ to CSQ$(h,\tau)$ is inconsistent with at most
$|A| + |A'| \le 2d/(d \tau^2-1)$ of the functions $f_i$.  Since $d$ (or, technically, $d-1$) functions need to be eliminated, this implies the desired lower bound.
\end{proof}

We then get the following as an immediate corollary.\\

\begin{corollary}
Let $C$ be a class with $\sqdim_D(C) = \omega(n^k)$ for all $k$, then $C$ is not
efficiently SQ-learnable under $D$.
\end{corollary}

Perhaps surprisingly, for distribution-specific learning, CSQ-learnability is equivalent to SQ-learnability.\\

\begin{lemma}[\citet*{BshoutyF02}] Any SQ 
can be answered by asking two SQs that are independent of the target and two CSQs.
\end{lemma}
\begin{proof}
We decompose the SQ into two SQs:
\begin{align}
\E_D[q(x,c(x)] = & \ \E_D\left[q(x,-1)\frac{1-c(x)}{2}+q(x,1)\frac{1+c(x)}{2} \right] \nonumber \\
=&\ \frac{1}{2}\E_D[q(x,1)c(x)] \ - \  \frac{1}{2}\E_D[q(x,-1)c(x)] \label{eq:corr} \ + \\
&\  \frac{1}{2}\E_D[q(x,1)] \ + \ \frac{1}{2}\E_D[q(x,-1)] \label{eq:nolabel}.
\end{align}
Note that the terms in Expression~\ref{eq:corr} are correlational statistical queries and the terms 
in Expression~\ref{eq:nolabel} are statistical queries independent of the label.
\end{proof}

On the other hand, \citet*{Feldman11} showed that CSQs are strictly weaker than SQs for distribution-independent
learning. For example, he showed that half-spaces are not distribution-independently CSQ learnable, but are SQ learnable.


There also exists a similar theorem for honest statistical queries, as given below.  The statement was originally proven
by \citet*{Yang05} and later strengthened by \citet*{FeldmanGRVX17}.\\

\begin{theorem}[\citet*{Yang05}]
Let $C$ be a concept class and let $d = \sqdim(C)$. Then any HSQ learning
algorithm must use a total sample complexity at least $\Omega(d)$ to learn $C$ to constant
accuracy and probability of success.
\end{theorem}

\subsubsection{Classes that are not efficiently SQ learnable}\label{sec:notsq}

Given the statistical query dimension lower bounds, we can now say certain classes
of functions are not learnable with statistical queries, begging with a result from the
results in the original paper of \citet*{Kearns98}. \\

 \begin{observation}\label{obs:sqdimparity}
 Parity functions on $\{0,1\}^n$ have  $\sqdim = 2^n$, and therefore, are not efficiently SQ learnable.
\end{observation}
 
Parity functions are of the form $\chi_c(x) = (-1)^{c \cdot x}$.  All $2^n$ of them are pairwise orthogonal.
This is known from orthogonality of Fourier characters under the uniform distribution; see the book by \citet*{ODonell14}.
Parities, however, being linear functions, are PAC-learnable using Gaussian elimination, so SQ $\subsetneq$ PAC.
\citep*{BlumFJKMR94}.\\

\begin{observation}
Decision trees on $n$ nodes have $\sqdim \ge n^{c \log n}$, and therefore, are not efficiently SQ learnable.
\end{observation}

This fact can be proven by showing how decision trees can encode many parity functions, all of which are pairwise orthogonal.
This is the standard technique for showing a high statistical query dimension.

 \begin{figure}[h]
 \begin{center}
\includegraphics[scale=0.25]{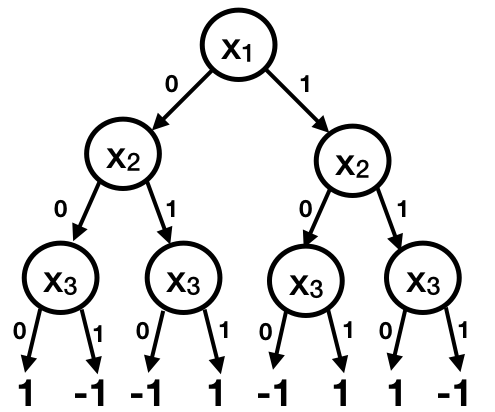} 
\end{center}
\caption{A decision tree on $7$ nodes encoding the parity function on $3$ variables.}\label{fig:SQtree}
\end{figure}

Figure~\ref{fig:SQtree} illustrates the straightforward way how a decision tree with $n-1$ nodes can encode a parity function on $\log n$ variables.
Since there are $\binom{n}{\log n}$ choices of $\log n$ from $n$ variables, this shows decision trees have a statistical query dimension of at least $n^{c \log n}$.\\

 \begin{observation}
 DNF of size n have $\sqdim \ge  n^{c \log n}$, and therefore, are not efficiently SQ learnable. 
 \end{observation} 

\begin{figure}[h]
\begin{center}
$(x_1 \wedge x_2 \wedge x_3) \vee (\bar{x}_1 \wedge \bar{x}_2 \wedge \bar{x}_3)
\vee (\bar{x}_1 \wedge {x}_2 \wedge \bar{x}_3) \vee (\bar{x}_1 \wedge \bar{x}_2 \wedge {x}_3)$
\end{center}
\caption{A $4$-term DNF encoding the parity function on $3$ variables}\label{fig:SQDNF}
\end{figure}

DNF formulae of size $n$ can similarly encore parity functions on $\log n$ variables using $n/2$ terms, as illustrated 
in Figure~\ref{fig:SQDNF}.\\

 \begin{observation}
 Deterministic finite automata on $n$ nodes have $\sqdim \ge 2^{cn}$, and therefore, are not efficiently SQ learnable.
 \end{observation} 

\begin{figure}[h]
\begin{center}
\includegraphics[scale=0.25]{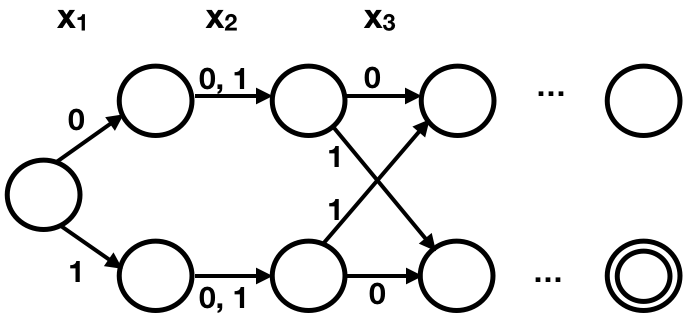} 
\end{center}
\caption{A DFA on $2n+1$ nodes encoding parities of size $n$.}
\end{figure}\label{fig:SQdfa}

Figure~\ref{fig:SQdfa} illustrates how deterministic finite automaton with $2n+1$ nodes can encode a parity function on $n$ variables.
Note that the crossings correspond to variables relevant to the parity function.

It turns out that even \emph{uniformly random} decision trees, DNF, and automata \citep*{AngluinEKR10}.

Note that only the first of these are known to be PAC learnable. We will discuss the implications of this in Section~\ref{sec:algs}.

\subsubsection{Comparison with VC dimension}

The results above imply certain relationships to other notions of dimension.  In this section, we briefly explore the relationship of $\sqdim$ with the 
Vapnik-Chervonenkis dimension, $\vcdim$, which controls the sample complexity of PAC learning \citep*{VapnikC15}.
Briefly stated, the $\vcdim$ of a concept class $\mathcal{C}$ is the maximum number of examples that $C$ can \emph{shatter}, i.e.\ achieve all possible 
labelings of the examples by functions in $C$.

First, we make the following observation, which also appears in~\citet*{BlumFJKMR94}\\
\begin{observation}
For a concept class $C$, let $\vcdim(C) = d$, then  $\sqdim(C) = \Omega(d)$.
\end{observation}
\begin{proof}
Let $d$ be the VC dimension of $C$.  Then there exists a set $S$ of $d$ points $C$ can shatter.  Assume without loss of generality that the domain
of $S$ is $\{0,1\}^{\log d}$. Because $C$ shatters $S$, it contains all $d$ parity functions over $\{0,1\}^{\log d}$, which by Observation~\ref{obs:sqdimparity} have $\sqdim$ of $2^{\log d} = d$.
\end{proof}

On the other hand, SQ dimension can be much larger than VC dimension.\\

\begin{observation}
There exist classes $C$, for which $\vcdim(C) = d$, but for which $\sqdim(C)$ can be as large as $2^d$.
\end{observation}

\begin{proof}
Parity functions on $\{0,1\}^d$ have $\vcdim=d$ but again by Observation~\ref{obs:sqdimparity} have $\sqdim = 2^d.$
\end{proof}

Finally, we might ask if there are classes with VC dimension $d$ but even larger SQ dimension.  The answer turns out to be no.\\

\begin{theorem}[\citet*{Sherstov18}]
Let $C$ be a concept class with $\vcdim(C) = d$.  Then, $\sqdim(C) \le 2^{O(d)}.$
\end{theorem}

\subsection{SQ upper bounds}

Following from Definition~\ref{def:sqdim} (statistical query dimension), we can also get an upper bound on the number of statistical
queries needed to achieve weak learnability.\\

\begin{observation}
Let $C$ be a concept class and let $\sqdim_D(C) = \poly(n)$, then $C$ is weakly
learnable under~$D$.
\end{observation} 
\begin{proof}
Let $S = \{f_1, \ldots f_d\} \subseteq C$ realize the SQ bound.  For each $f_i \in S$,
query its correlation with $c^*$.  At least one must have a correlation greater than $1/d$; otherwise we could add
$c^*$ to $S$, contradicting $S$'s maximality.
\end{proof}

Because of this observation, $\sqdim$ is sometimes referred to as the \emph{weak} statistical query dimension.

One may then  ask about strong learnability, as in Definition~\ref{def:sqlearning} (efficient SQ learning).
\citet*{Schapire90} showed that a class is strongly learnable if and only if it is weakly learnable in the PAC setting.  
It is then natural to ask whether the same equivalence between weak and strong learnability 
holds in the SQ setting, and indeed \citet*{AslamD98b} showed ``statistical query boosting" is possible.\\

\begin{theorem}[\citet*{AslamD98b}]
Let $d = \sqdim(C)$, then $C$ is SQ-learnable to error $\epsilon > 0$ using $O(d^5 \log^2 ({1}/{\epsilon}))$  queries with tolerances bounded by
$\tau = \Omega({\epsilon}/{(3d)})$.
\end{theorem}

The outline of the proof of the above theorem is as follows: the learner simulates boosting by feeding in his series of
weighted weak learners to the SQ oracle via a statistical query and then asking the oracle to simulate the resulting distribution.

But this procedure, like regular boosting, works only for \emph{distribution independent} learning, i.e.\ when weak learnability is achievable for any distribution.
In the distribution-dependent case, (weak) SQ dimension does not necessarily
characterize strong learnability.

For this reason, there exist definitions for a corresponding notion of strong SQ dimension
\citep*{Feldman12,Simon07,Szorenyi09}.  We provide a definition here;
roughly speaking, $\ssqdim_D(C,1-\epsilon)$, controls the complexity of learning $C$.\\
\begin{definition}[strong statistical query dimension]
For a concept class $C$ and distribution $D$, let the \underline{strong statistical  query dimension} $\ssqdim_D(C,\gamma)$ be the largest $d$ such that some $f_1, \ldots, f_d \in C$
fulfill
\begin{enumerate}
\item $|\left<f_i,f_j\right>_D| \le \gamma$ for $1 \le i < j \le d$, and
\item $|\left<f_i,f_j\right>_D - \left<f_k,f_\ell\right>_D| \le 1/d$ for $1 \le i < j \le d$, $1 \le k < \ell \le d.$
\end{enumerate}
\end{definition}


For $\epsilon = 1/10$, the gap between strong and weak SQ dimension can be as large as possible.
To see this, 
consider the following class of functions:
$$\mathcal{F} = \{v_1 \vee \chi_c\ |\ c \in \{0,1\}^n\}.$$ 
Then it is not hard to see that  $\sqdim_U(\mathcal{F}) = 1$ but $$\ssqdim_U(\mathcal{F},9/10) = 2^n.$$

\citet*{Feldman12} also showed that a variant of $\ssqdim$ captures the complexity of 
agnostic learning of a hypothesis class, which implies that even agnostically learning conjunctions
is not possible with statistical queries


\subsection{The complexity of learning}\label{sec:algs}

If we consider SQ, PAC, etc.\ as classes that contain classes of functions that are learnable in those respective models, 
we have seen that
$$ \mathrm{efficient\ SQ}\ \subseteq\ \mathrm{efficient\ PAC\ under\ classification\ noise}\ \subseteq\ \mathrm{efficient\ PAC}.$$
In Section~\ref{sec:notsq}, we have also seen that parity functions are efficiently PAC learnable, but not efficiently SQ learnable.
So, a natural question is whether
parity functions are learnable in PAC under classification noise?  This question is the (notorious) problem of \emph{learning parities under noise} (LPN).

There was indeed some progress on the LPN problem.
\citet*{BlumKW03} gave a $2^{O(n/\log n)}$ algorithm for efficiently learning parities in PAC under (constant) classification noise.
 This implies that the (admittedly artificial) class of parities on the first $k = \log n \log \log n$ bits
are efficiently learnable in PAC under classification noise, but not efficiently SQ learnable.

It is, however, widely believed that there is no efficient algorithm for the LPN problem in general. Variants have been proposed for public-key cryptography \citep*{Peikert14}.
There has been some progress on this and related problems, but we are far from efficient algorithms. \citep*{BlumKW03,GrigorescuRV11,Valiant15}).


A series of results has show how to implement analogues\footnote{While, for example, stochastic
gradient descent is technically not a statistical algorithm, a noisy variant of it can be implemented
via a statistical query oracle.} of many of current algorithmic approaches via a statistical query oracle.  These include
\begin{itemize}
\item Gradient descent \citep*{Robbins51}
\item Expectation-maximization (EM) \citep*{DempsterLR77}
\item Support vector machines (SVM) \citep*{CortesV95,MitraMP04}
\item Linear and convex optimization~\citep*{DunaganV08}
\item Markov-chain Monte Carlo (MCMC)~\citep*{TannerW87,GelfandSmith90}
\item Simulated annealing~\citep*{Cerny85,KirkpatrickGV83}
\item Pretty much everything else, including PCA, ICA, Na\"ive Bayes, neural net algorithms, $k$-means~\citep*{BlumDMN05}.
\end{itemize}

On the other hand, we have only few algorithms that have no analogous implementation via a statistical
query oracle. 
These include variants of
Gaussian elimination, hashing, and bucketing.
Most of our other techniques seem to be implementable with statistical queries.
This helps explain why we don't have algorithms for many natural 
classes, including decision trees and DNF, which have high SQ dimension and are therefore difficult to learn
using current techniques even in the absence of noise.

To tackle these problems, it appears we need to invent fundamentally different methods.


\section{Applications}\label{sec:applications}

In this section, we explore three modern applications of statistical queries.  These include optimization
problems over distributions, evolvability and differential privacy / adaptive data analysis.  We conclude
with a small collection of other areas to show the diversity of the applications of statistical queries.

\subsection{Optimization and search over distributions}

As a motivating example of an optimization problem over a distribution, consider the problem of finding the direction that maximizes the $r$th moment over a distribution $D$, 
$$\mathrm{argmax}_{u : |u|=1} \E_{x \sim D}[(u \cdot x)^r].$$
For $r=1$, this is maximized at the mean, which is easy to compute.  For $r=2$, we need the direction of highest variance, and 
PCA gives the solution.  For $r \ge 3$, these are
strong complexity and information-theoretic reasons to think this moment maximization problem is intractable.

Statistical algorithms apply to such optimization problems over distributions.  In this setting, there is a distribution $D$ unknown to the learner,
and the learner would normally try to solve such optimization problems by working over a sample from $D$.

Carrying over the statistical query ideas from learning, \citet*{FeldmanGRVX17} extended this setting to search and optimization 
problems over distributions. 
Any problem with instances coming from a distribution $D$ (over $X$) can be analyzed via a \emph{statistical oracle}, which
is meant to be a generalization of a statistical query oracle to settings without labels.  

They defined three oracles: STAT, which 
corresponds to the SQ oracle; $1$-STAT, which corresponds to an HSQ oracle working over $1$ sample at a time; and VSTAT, which
corresponds to the range of results expected from an independent sampling procedure from a Bernoulli distribution with a given mean.\\

\begin{definition}[The STAT, $1$-STAT, and VSTAT oracles]
Let $q: X \rightarrow \{0,1\}$, $\tau > 0$ a tolerance, and $t > 0$ a sample size.
\begin{itemize}
\item \underline{STAT}$(q,\tau)$:\ \ \ \ \ \hskip.04in returns a value in: $[\mu- \tau, \mu  + \tau],$
\item \underline{$1$-STAT}$(q)$:\ \ \ \ \ \ \hskip.01in draws $1$ sample, $x \sim D$, and returns $q(x)$,
\item \underline{VSTAT}$(q,m)$:\ \ \ returns a value $[\mu - \tau', \mu  + \tau'],$
\end{itemize}
where $\mu = \E_{x \sim D}[q(x)]$ and $\tau' = \max \left\{1/m, \sqrt{\mu(1-\mu)/m} \right\}$. 
\end{definition}

\subsubsection{Statistical dimension}

Like the notion of statistical query dimension,~\citet*{FeldmanGRVX17} defined an analogous distributional 
notion called statistical dimension.  The notion that they use involves a stronger notion of \emph{average} correlation, but 
we first need to define the pairwise correlation of two distributions.\\

\begin{definition}[pairwise correlation of two distributions] 
Define the \underline{pairwise correlation} of $D_1$, $D_2$ with respect to $D$ is
$$\chi_D(D_1,D_2) = \left| \left< \frac{D_1}{D}-1, \frac{D_2}{D}-1 \right>_D \right|.$$
Note that $\chi_D(D_1,D_1) = \chi^2(D_1,D)$, the chi-squared distance between $D_1$ and $D$ \citep*{Pearson00}.
\end{definition}

As an example of the definition above, let $X = \{0,1\}^n$ and $D_{c_1}, D_{c_2}$ be uniform over the examples labeled $-1$ by $\chi_{c_1}, \chi_{c_2}$, 
respectively. It turns out $\chi_U(D_{c_1},D_{c_2}) = 0$.

To see this, let us compute  $ \chi_U(D_{010},D_{011})= \left< \frac{D_{010}}{U}-1, \frac{D_{011}}{U}-1 \right>_U $
 for $n=3$ using the table below.

\begin{center} 
\begin{tabular}{|c|c|c|c|c|c|c|c|}
\hline
$X$ &  $U$ & $D_{010}$ & $D_{011}$ &  $\frac{D_{010}}{U}$ &  $\frac{D_{011}}{U}$ & $\frac{D_{010}}{U}-1$ & $\frac{D_{011}}{U}-1$\\
\hline
{000} & 1/8 & 0 & 0 & 0 & 0 & -1 & -1\\
001 & 1/8 & 0 & 1/4 & 0 & 2 & -1 & 1\\
010 & 1/8 & 1/4 & 1/4 & 2 & 2 & 1 & 1\\
011 & 1/8 & 1/4 & 0 & 2 & 0 & 1 & -1\\
100 & 1/8 & 0 & 0 & 0 & 0 & -1 & -1\\
101 & 1/8 & 0 & 1/4 & 0 & 2 & -1 & 1\\
110 & 1/8 & 1/4 & 1/4 & 2 & 2 & 1 & 1\\
111 & 1/8 & 1/4 & 0 & 2 & 0 & 1 & -1\\
\hline

\end{tabular}
\end{center}
\begin{align*} 
\left< \frac{D_{010}}{U}-1, \frac{D_{011}}{U}-1 \right>_U = & 
\ \ \frac{(-1)(-1)}{8} \ + \ \frac{(-1)(1)}{8} \ + \ \frac{(1)(1)}{8} \ + \ \frac{(1)(-1)}{8} \ + \\
& \ \ \frac{(-1)(-1)}{8} \ + \ \frac{(-1)(1)}{8} \ + \ \frac{(1)(1)}{8} \ + \ \frac{(1)(-1)}{8} \\ = & \ \ 0
\end{align*}


Now we define another and stronger notion called average correlation.\\


\begin{definition}[average correlation of a set of distributions]
Define the \underline{average correlation} of a set of distributions $\mathcal{D}'$ relative to $D$ as
$$
\rho(\mathcal{D'},D) = \frac{1}{|\mathcal{D'}|^2}\sum_{D_1,D_2 \in \mathcal{D}'} \chi_D(D_1,D_2).
$$
\end{definition}

Now, we can finally define statistical dimension with average correlation (SDA).\\

\begin{definition}[statistical dimension with average correlation\footnote{We chose to use this definition of statistical dimension  because it was the framework
in which the first novel optimization lower bound (on the planted clique problem, as presented in Section~\ref{sec:plantedclique}) was proven, 
and because the survey's aim to illustrate the application as opposed to giving the tightest possible bounds here.
However, statistical dimension with average correlation does not always give the strongest lower bounds, and it was
later strengthened to use discrimination norm \citep*{FeldmanPV15} and then
extended to ``Randomized Statistical Dimension" \citep*{Feldman17}.}]
For $\bar{\gamma} > 0$, a domain $X$, a set of distributions $\mathcal{D}$ over $X$ and a reference distribution $D$ over $X$, the \underline{statistical dimension} of $\mathcal{D}$ relative to $D$ \underline{with average correlation}
 $\bar{\gamma}$ is defined to be the largest value $d$ such that for any subset $\mathcal{D}' \subseteq \mathcal{D}$ for which $|\mathcal{D'}| \ge \mathcal{D}/d$, we have $\rho(\mathcal{D}',D) \le \bar{\gamma}$. This is denoted SDA$_D(\mathcal{D}, \bar{\gamma})$. 
For a search problem $\mathcal{Z}$ over distributions\footnote{
The definition of search problems, as given by~\citet*{FeldmanGRVX17}, is as follows: for a domain $X$,
land $\mathcal{D}$ a set of distributions over $X$, let $F$ be a set called {\em solutions} and $\mathcal{Z}: \mathcal{D} \rightarrow 2^{\mathcal{F}}$   be a map from a distribution $D \in \mathcal{D}$ to a subset of solutions $\mathcal{Z}(D)
\subseteq \mathcal{F}$ that are defined to be valid solutions for $D$.
 The search problem $\mathcal{Z}$ over $\mathcal{D}$ and $F$ using $t$ samples is to find a valid solution $f \in 
 \mathcal{Z}(D)$ given access to an unknown $D \in \mathcal{D}$.
}, we use:
SDA$(\mathcal{Z},\bar{\gamma})$. 
\end{definition}

Intuitively, the largest such $d$ for which $1/d$ fraction of the set of distributions has
low pairwise correlation is the statistical dimension.\\

\begin{theorem}[\citet*{FeldmanGRVX17}]\label{thm:sdabd}
Let $X$ be a domain and $\mathcal{Z}$ be a search problem over a 
class of distributions $D$ over $X$. For $\bar{\gamma} > 0$, let $d = \mathrm{SDA}(\mathcal{Z}, \bar{\gamma})$.
To solve $\mathcal{Z}$ with probability $\ge 2/3$, any SQ algorithm requires at least:
\begin{itemize}
\item $d$ calls to VSTAT$(.\ , c_1/\bar{\gamma})$
\item $\min(d/4, c_2/\bar{\gamma} )$ calls to $1$-STAT$(.)$
\item $d$ calls to STAT$(.\ , c_3\sqrt{\bar{\gamma}})$.
\end{itemize}
\end{theorem}

The proof by~\citet*{Szorenyi09} of the weaker version of Theorem~\ref{thm:sqdimlb} (of the SQ-DIM lower bound for CSQs) 
gives the intuition for this claim,
where we can observe that the result in Equation~\ref{eq:szorenyibd} can be derived so long as the 
average correlation between $f_i,f_j \in A$ is bounded, where $A$ is a large enough set of functions.


We note the many differences  from (or extensions to) SQ-DIM. First, this model has no need for labels.
Second, the notion of correlation is denoted not by $\gamma$ but rather by by $\bar{\gamma}$, which stands for
\emph{average} (not worst-case) correlation.  Third, $d$ is disconnected from $\bar{\gamma}$ in the definition.
And finally a new type of oracle (VSTAT) is considered.

The main application of this model is to give lower bounds for new types of problems.  In the next section
we give the lower bound provided by~\citet*{FeldmanGRVX17} for the planted clique problem.

\subsubsection{Planted clique: an application of statistical dimension}\label{sec:plantedclique}

Consider the long-standing \emph{planted clique} problem, introduced by \citet*{Jerrum92}, of detecting a $k$-clique randomly induced
in a $G(n,\frac{1}{2})$ Erd\"{o}s-R\'{e}nyi random random graph instance.
Information-theoretically, this is possible for $k > 2\log(n)+1$, but
the state-of-the-art polynomial-time algorithm~\citep*{AlonKS98} uses spectral techniques to recover cliques of size $k > \Omega(\sqrt{n})$.
For the last two decades, this bound has eluded improvement.

Statistical algorithms help to explain why.  SDA lower bounds show that statistical algorithms cannot efficiently recover cliques of size $O(n^{1/2-\epsilon})$.
To use the SDA machinery, we first need to define a distributional version of planted clique.\\

\begin{problem}[distributional planted $k$-biclique]
For k, $1 \le k \le n$, and a subset of $k$ indices $S \subseteq \{1, 2,\ldots, n\}$.
The input distribution $D_S$ on vectors $x \in \{0,1\}^n$ is defined as follows: w.p.\
$1 -k/n$, $x$ is uniform over $\{0,1\}^n$; and w.p.\ $k/n$, $x$ is such that its $k$
coordinates from $S$ are set to $1$, and the remaining coordinates are uniform in $\{0, 1\}$.
The problem is to find the unknown subset $S$.
\end{problem}

An example is given in Figure~\ref{fig:biclique}.

\begin{figure}[h]
\begin{center}
\includegraphics[scale=0.25]{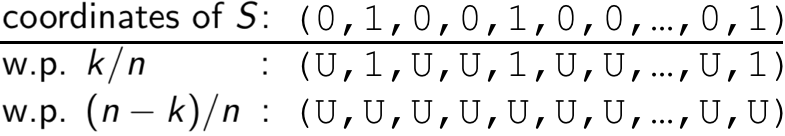}
\end{center}
\caption{An example distributional planted biclique instance.}\label{fig:biclique}
\end{figure}

Now we can analyze the statistical dimension for the planted clique problem.\\

\begin{theorem}[\citet*{FeldmanGRVX17}]
For $\epsilon \ge 1/\log n$ and $k \le n^{1/2-\epsilon}$, let $\mathcal{D}$ be the set of all planted $k$-clique
distributions.  Then $$\mathrm{SDA}_U(\mathcal{D},2^{\ell+1}k^2/n^2) \ge n^{2\ell\delta}/3.$$
\end{theorem}

Using Theorem~\ref{thm:sdabd}, we can get the following lower bound on the number of queries as a corollary of the above result.
For simplicity, we only give the lower bound for the VSTAT oracle, which is the strongest of the lower bounds, below.\\

\begin{corollary}[\citet*{FeldmanGRVX17}]
For any constant $\epsilon > 0$ and any $k \le n^{1/2-\epsilon}$ , and $r > 0$,
to solve distributional planted $k$-biclique with probability $\ge 2/3$,
any statistical algorithm requires
at least $n^{\Omega(\log r)}$
queries to VSTAT$(.\ ,n^2/(rk^2))$.
\end{corollary}

An interpretation of this bound says that we would need an exponential number of queries
of the precision that a ``sample size" of $n$ would give us, which is all we get in a ``real-world" planted-clique instance.

\subsection{Evolvability}

Statistical queries can also help to better understand biological evolution as an algorithmic process.

\citet*{Valiant09} defined the evolvability framework to model and formalize Darwinian evolution, with
the goal of understanding what is ``evolvable."  This requires some definitions, and 
we begin with the most basic concept of an evolutionary algorithm.\\

\begin{definition}[evolutionary algorithm]
An \underline{evolutionary algorithm} $A$ is defined by a pair $(R, M)$ where
\begin{itemize}
\item $R$, the representation, is a class  of functions from $X$ to $\{-1,1\}$.
\item $M$, the mutation, is a randomized algorithm that, given $r \in R$ and an $\epsilon > 0$, outputs an $r' \in R$ with probability $\Pr_A(r,r')$.
\end{itemize}
$\mathtt{Neigh}_A(r,\epsilon) =$ set of $r'$ that $M(r,\epsilon)$ may output (w.p.\ $1/p(n,1/\epsilon)$).
\end{definition}

Then we define the notion of a performance of a given representation with respect to an ideal function (that we are trying to evolve or applroximately
evolve).\\

\begin{definition}[performance and empirical performance] Let  $f: X \rightarrow \{-1,1\}$ be an ideal function.\\
The \underline{performance} of $r \in R$ with respect to $f$ is
$$\mathtt{Perf}_{f,D}(r) = \E_{x \sim D} [f(x)r(x)].$$
The \underline{empirical performance} of $r$ on $s$ samples ${x_1, \ldots, x_s}$  from $D$ is 
$$\mathtt{Perf}_{f,D}(r,s) = \frac{1}{s}\sum_{i}^{t} f(x_i)r(x_i).$$
\end{definition}

And as in biological evolution, in this model, selection operates on the representations to produce the next
generation of representations.\\

\begin{definition}[selection]
\underline{Selection} $\mathtt{Sel}[\tau,p,s](f,D,A,r)$ with parameters: tolerance $\tau$, pool size $p$, and sample size $s$ operating
on  $f,D,A=(R,M),r$ defined as before, outputs $r^+$ as follows.
\begin{enumerate}
\item Run $M(r,\epsilon)$ $p$ times and let $Z$ be the set of $r'$s obtained.  
\item For $r' \in Z$, let $\Pr_Z(r')$ be the frequency of $r'$. 
\item For each $r' \in Z \cup \{r\}$ compute $v(r') = \mathtt{Perf}_{f,D}(r',s)$
\item Let $\mathtt{Bene}(Z) = \{r'\ |\ v(r') \ge v(r) + \tau\}$ and 
$\mathtt{Neut}(Z) = \{r'\ |\ |v(r') - v(r)| + \tau\}$
\item
 if $\mathtt{Bene} \neq \emptyset,$\ \hskip.32in output $r^+$ proportional to $\Pr_Z(r^+)$ in $\mathtt{Bene}$ \\
 else if $\mathtt{Neut} \neq \emptyset,$\ \ output $r^+$ proportional to $\Pr_Z(r^+)$ in $\mathtt{Neut}$ \\
 else \hskip.81in output $\perp$
\end{enumerate}
\end{definition}

This lets us define what we mean by a function class being evolvable by an algorithm.\\

\begin{definition}[evolvability by an algorithm]
For concept class $C$ over $X$, distribution $D$, and evolutionary algorithm $A$,
we say that the class $C$ is \underline{evolvable over $D$ by $A$}
if there exist polynomials, $\tau(n, 1/\epsilon)$, $p(n, 1/\epsilon)$, $s(n, 1/\epsilon)$, and  $g(n, 1/\epsilon)$ such that for every $n$, $c^* \in C$, $\epsilon > 0$, and every 
$r_0 \in R$, with probability at least $1-\epsilon$, the random 
sequence $r_i \leftarrow \mathtt{Sel}[\tau,p,s](c^*,D,A,r_{i-1})$ will yield a $r_g$ s.t.\
$\mathtt{Perf}_{c^*,D}(r_g) \ge 1-\epsilon$.
\end{definition}

Finally, we can define evolvability of a concept class.\\

\begin{definition}[evolvability of a concept class]
A concept class $C$ is \underline{evolvable} (over $\mathcal{D}$) if
there exists an evolutionary algorithm $A$ so that for any 
for any $D (\in \mathcal{D})$ over $X$, $C$ is evolvable over $D$ by $A$.
\end{definition}

%
%
%
%
%
%
%
%
%
%
%
%
%
%
%
%
%
%


The main result here is that it turns out that evolvability is equivalent to learnability with CSQs, as stated 
below.\\

\begin{theorem}[\citet*{Feldman08}]
$C$ is evolvable if and only if $C$ is learnable with CSQs (over $\mathcal{D}$).
\end{theorem}

That EVOLVABLE $\subseteq$ CSQ is immediate~\citet*{Valiant09}. 
The other direction involves first showing that 
$$\mathrm{CSQ}_{>}(r,\theta,\tau) = 
\begin{cases}
 1 \mathrm{\ \ \ \ \ \ \ \ \ \ \ \ if\ } \E_D[r(x)c^*(x)] \ge \theta + \tau\\
 0 \mathrm{\ \ \ \ \ \ \ \ \ \ \ \ if\ } \E_D[r(x)c^*(x)] \le \theta -\tau\\
 0\ \mathrm{or}\ 1\ \ \ \ \ \ \mathrm{otherwise}
 \end{cases} $$
 can simulate CSQs.
 Then an evolutionary algorithm is made that simulates queries to a CSQ$_>$ oracle.

\subsubsection{Sexual evolution}

Valiant's model of evolvability is asexual.  \citet*{Kanade11} extended evolvability to include recombination by replacing
 $\mathtt{Neigh}$ (neighborhood) with $\mathtt{Desc}$ (descendants).\\

\begin{definition}[recombinator]

For polynomial $p(·, ·)$, a $p$-bounded
\underline{recombinator} is a randomized algorithm that
takes as input two representations $r_1, r_2 \in R$ and $\epsilon$ 
and outputs a set of representations $\mathtt{Desc}(r_1,r_2,\epsilon) \subseteq R.$
Its running time is bounded by $p(n,1/\epsilon)$.
$\mathtt{Desc}(r_1,r_2,\epsilon)$
is allowed to be empty which is interpreted as $r_1$ and $r_2$
being unable to mate.
\end{definition}

Now we can examine evolution under recombination.\\

\begin{definition}[parallel CSQ]
A \underline{parallel CSQ} learning algorithm uses $p$ (polynomially bounded) processors and we 
assume that there is a common clock which defines
parallel time steps. During each parallel time step a processor
can make a CSQ query, perform polynomially-bounded
computation, and write a message that can be read by every
other processor. We assume that communication happens at
the end of each parallel time step and on the clock. The CSQ
oracle answers all queries in parallel.
\end{definition}

Sexual evolution is equivalent to parallel CSQ learning.\\

\begin{theorem}[\citet*{Kanade11}]
If $C$ is parallel CSQ learnable in $T$
query steps, then $C$ is evolvable under recombination in
$O(T \log^2(n/\epsilon))$ generations.
\end{theorem}

\subsection{Differential privacy and adaptive data analysis}

Our final application is to differentially private learning and to adaptive data analysis,
both of which are closely connected to each other.

\subsubsection{Differentially private learning}

The differential privacy of an algorithm captures an individual's ``exposure" of being in a database 
when that algorithm is used \citep*{DworkMNS06}.\\

\begin{definition}[differential privacy]
A probabilistic mechanism $\mathcal{M}$ satisfies $(\alpha,\beta)$-\underline{differential privacy}\footnote{Oftentimes,
$\epsilon$ and $\delta$ are used to define differential privacy.  We instead use $\alpha$ and $\beta$ so as to not
confuse these variables for the $\epsilon$ and $\delta$ parameters in PAC learning.} 
 if for any two samples
$S, S'$ that differ in just one example, for any outcome $z$
$$\Pr[\mathcal{M}(S) = z] \le e^\alpha \Pr[\mathcal{M}(S') =z]+\beta.$$
If $\beta=0$, we simply call $\mathcal{M}$ $\alpha$-differentially private.
\end{definition}

\begin{figure}[h]
\begin{center}
\includegraphics[scale=0.30]{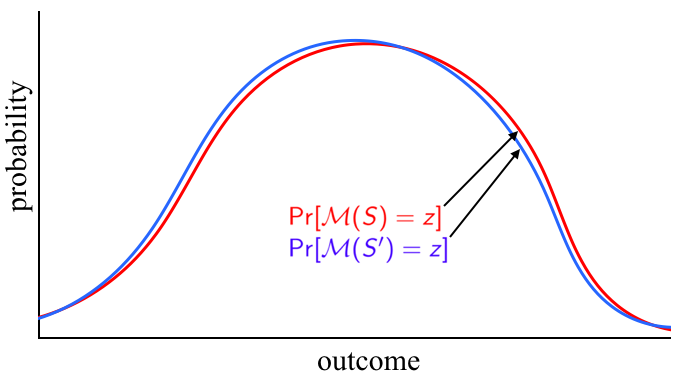}
\end{center}
\caption{An illustration of a possible output of a differentially private mechanism.  Here, the datasets $S$
and $S'$ differ by one example and their respective red and blue distributions over outputs differ by an amount
that is bounded by the parameters $\alpha$ and $\beta$.}
\end{figure}

We now define the Laplace mechanism, which can be used to guarantee differential privacy.\\

\begin{definition}[Laplace mechanism]
Given $n$ inputs in $[0,1]$, the \underline{Laplace mechanism} for outputting their
average computes the true average value $a$ and then outputs $a + x$ where $x$ is drawn from the
Laplace density with parameters $\left(0, \frac{1}{\alpha n}\right)$:
$$\mathrm{Lap}_{(0,\frac{1}{\alpha n})}(x) = \left(\frac{\alpha n}{2}\right)e^{-|x|\alpha n}.$$
\end{definition}

\begin{theorem}[\citet*{DworkMNS06}]
The Laplace mechanism satisfies $\alpha$-differential privacy, and moreover has the property
that with probability $\ge 1 - \delta$, the error added to the true average is
$O\left(\frac{\log (1/\delta')}{\alpha n}\right)$.
\end{theorem}

It turns out the statistical queries are a perfect class of functions for applying the Laplace
mechanism, which gives the result below.\\

\begin{theorem}[\citet*{DworkMNS06}]
If class $C$ is efficiently SQ learnable, then it is also efficiently
PAC learnable while satisfying $\alpha$-differential privacy, with time and
sample size polynomial in $1/\alpha$. In particular, if there is an algorithm that makes $M$ queries of
tolerance $\tau$ to learn $C$ to error $\epsilon$ in the SQ model, then a sample of size
$$
m = O\left(\left[\frac{M}{\alpha \tau} + \frac{M}{\tau^2}\right]\log\left(\frac{M}{\delta}\right)\right)
$$
is sufficient to PAC learn $C$ to error $\epsilon$ with probability $1-\delta$ while satisfying $\alpha$-differential
privacy.
\end{theorem}

This is achieved by taking large enough sample and adding Laplace noise with scale parameter as to satisfy 
$\frac{\alpha}{M}$-differential privacy per query while staying within $\tau$ of the expectation of each query.

As we have seen SQ learnability is a sufficient condition for differentially-private learnability, but it is not a
necessary one.  It turns out, however, that
information-theoretically, SQ learnability is equivalent to a more restricted notion of privacy called
local differential privacy.

%

Informally, local differential privacy asks not only the output of the mechanism to be differentially private but
also the data itself to be differentially private with respect to the (possibly untrusted) mechanism.  Hence,
noise needs to be added to the data itself.

We state the connection below without providing details.\\

\begin{theorem}[\citet*{KasiviswanathanLNRS11}]
Concept class $C$ is locally differentially privately learnable if and
only if $C$ is learnable using statistical queries.\footnote{This result is only known for
the sequentially interactive model of local differential privacy.  The analogous
question for fully interactive differential privacy is still open.~\citep*{JosephMR20}}
\end{theorem}

\subsubsection{Adaptive data analysis}
Interestingly, differential privacy has applications to an even newer of study called \emph{adaptive data analysis}, which 
was introduced by \citet*{DworkFHPRR15}.

The main question in this area asks to what extent it is possible to answer adaptive queries accurately given a sample without 
assuming anything about their complexity (e.g.\ without limiting their VC dimension or Rademacher complexity).\\



\begin{definition}[adaptive accuracy]
A mechanism $\mathcal{M}$ is \underline{$(\alpha,\beta)$-accurate} on a distribution $D$ and on queries $q_1,\ldots,q_k$, 
if for its responses $a_1,\ldots a_k$ we have
$$\Pr_{\mathcal{M}} [\max|q_i(D)-a_i|\le \alpha] \ge1-\beta.$$
\end{definition}
Note: there is also an analogous notion of $(\alpha,\beta)$ accuracy on a sample $S$.

A natural question
is how many samples from $D$ are needed to answer $k$ queries adaptively with $(\alpha,\beta)$-accuracy.
Because there is no assumption about the complexity of the class from which the $q_i$s come.  So, standard
 techniques don't apply.

Differential privacy, however, gives us the techniques needed to answer this question by providing a notion of stability that transfers to
guarantees of adaptive accuracy. 
The following is an example of such a \emph{transfer theorem}.\\
 
\begin{theorem}[\citet*{DworkFHPRR15}]
Let $\mathcal{M}$ be a mechanism that on sample $S \sim D^n$ 
answers $k$ adaptively chosen statistical queries, 
is $\left(\frac{\alpha}{64},\frac{\alpha \beta}{32}\right)$-private for
some $\alpha,\beta > 0$ and $\left(\frac{\alpha}{8}, \frac{\alpha \beta}{16}\right)$-accurate on $S$. Then $\mathcal{M}$ is
          $(\alpha, \beta)$-accurate on $D$.
\end{theorem}

Putting together the Laplace mechanism with the transfer theorem,
and doing some careful analysis to improve the bounds, one can get an adaptive
algorithm for SQs.\\

\begin{theorem}[\citet*{BassilyNSSSU16}]
There is a polynomial-time mechanism that is $(\alpha,\beta)$-accurate with respect to any
distribution $D$ for $k$ adaptively chosen statistical queries given
$$
m = \tilde{O} \left(\frac{\sqrt{k}\log^{3/2}(1/\beta)}{\alpha^2} \right)
$$
samples from $D$.
\end{theorem}

There are of course various improvements to this result. For example,
subsampling \citep*{KasiviswanathanLNRS11} can speed up the Laplace mechanism without
increasing the overall sample complexity of adaptive data analysis~\citep*{FishRR20}.

\subsection{Other applications}

While this survey has focused on the preceding three application areas, statistical queries have had impact in many other fields.
Here, we give a sampling of some statements of applications, leaving it to the interested reader to learn more about these results.

The first result concerns analyzing how the statistical query dimension of a concept class can separate two classes in communication 
complexity.\\
\begin{theorem}[\citet*{Sherstov08}]
Let $C$ be the
class of functions $\{-1, 1\}^n \rightarrow \{-1, 1\}$ computable in $\mathrm{AC}^\mathrm{0}$. 
If $$\sqdim(C) \le O\left(2^{2^{(\log n)^\epsilon}} \right)$$ for every constant $\epsilon > 0$, then $$\mathrm{IP} \in \mathrm{PSPACE}^\mathrm{cc} \setminus \mathrm{PH}^\mathrm{cc}.$$
\end{theorem}

Another application is distributed computing.  Here we state the following Theorem informally.\\
\begin{theorem}[\citet*{ChuKLYBNO06}]
SQ algorithms can be put into ``summation form" and  automatically parallelized in MapReduce,  giving
nearly-linear speedups in practice.
\end{theorem}

The final application we cover applies to streaming algorithms, relating the learnability of a class with statistical queries
to learnability from a stream. \\
\begin{theorem}[\citet*{SteinhardtVW16}]
Any class $\mathcal{C}$ that is learnable with $m$ statistical queries of tolerance $1/m$, it
is
learnable from a stream of $\poly(m, log |\mathcal{C}|)$ examples and $b = O(\log |\mathcal{C}| \log(m))$ bits of memory.
\end{theorem}

\section{Discussion and some open problems}

To summarize, we saw that statistical queries originate from a framework motivated, in part, for producing noise-tolerant algorithms.
However, it turns out that actually most of our algorithms can be (approximately) made to 
work in the statistical query framework, which
explains many of our impediments in learning and optimization.
Statistical queries have also had  applications that have shed light on the difficulty of other problems.
There were also perhaps unexpected applications, to differential privacy, adaptive data analysis, 
evolvability, among other areas.

It is perhaps appropriate to conclude with some open questions arising from the vast literature on statistical queries, some of which
this survey has not even covered.  One important but difficult direction is to find new and
clearly non-statistical approaches to the many problems for which statistical algorithms are known to fail due
to the lower bounds presented herein.

We will not attempt to give a comprehensive or even a long list of specific 
open questions across the various areas; rather, we will give a sampling.
Many questions are more technical -- for example, the ~\citet*{BlumKW03} result separating PAC
under classification noise only holds for constant noise rates -- can this be generalized to noise rates approaching $1/2$ as allowed by the
\citet*{AngluinL87} model?  Other directions include precisely determining 
the sample complexity of adaptively answering SQs  --
the strongest known lower bound, due to \citep*{HardtU14}, is $\Omega(\sqrt{k}/\alpha)$ and the upper bound, due to
\citep*{BassilyNSSSU16}, is $O(\sqrt{k}/\alpha^2)$.
In evolvability, we can ask about designing or analyzing faster or more natural algorithms for evolving functions (e.g.\ the swapping
algorithm~\citep*{DiochnosT09,Valiant09}.  In optimization, finding more problems, like planted clique, whose hardness
is explained by high statistical dimension is an active area.

But the most important (and very open-ended) 
question may lie in thinking more broadly about where else statistical
queries can have an impact.  
It is likely that they will find even more unexpected uses.

\subsection*{Acknowledgements}

This survey benefitted from helpful comments by
Daniel Hsu, Matthew Joseph, and Adam Klivans on its first draft.
This work was supported in part by grants CCF-1934915 and CCF-1848966 from the National
Science Foundation

\bibliography{SQ_survey}

\begin{thebibliography}{54}
\providecommand{\natexlab}[1]{#1}
\providecommand{\url}[1]{\texttt{#1}}
\expandafter\ifx\csname urlstyle\endcsname\relax
  \providecommand{\doi}[1]{doi: #1}\else
  \providecommand{\doi}{doi: \begingroup \urlstyle{rm}\Url}\fi

\bibitem[Alon et~al.(1998)Alon, Krivelevich, and Sudakov]{AlonKS98}
Noga Alon, Michael Krivelevich, and Benny Sudakov.
\newblock Finding a large hidden clique in a random graph.
\newblock \emph{Random Struct. Algorithms}, 13\penalty0 (3-4):\penalty0
  457--466, 1998.

\bibitem[Angluin and Laird(1987)]{AngluinL87}
Dana Angluin and Philip~D. Laird.
\newblock Learning from noisy examples.
\newblock \emph{Machine Learning}, 2\penalty0 (4):\penalty0 343--370, 1987.

\bibitem[Angluin et~al.(2010)Angluin, Eisenstat, Kontorovich, and
  Reyzin]{AngluinEKR10}
Dana Angluin, David Eisenstat, Leonid Kontorovich, and Lev Reyzin.
\newblock Lower bounds on learning random structures with statistical queries.
\newblock In \emph{Algorithmic Learning Theory, 21st International Conference,
  {ALT} 2010, Canberra, Australia, October 6-8, 2010. Proceedings}, pages
  194--208, 2010.

\bibitem[Aslam and Decatur(1998{\natexlab{a}})]{AslamD98a}
Javed~A. Aslam and Scott~E. Decatur.
\newblock Specification and simulation of statistical query algorithms for
  efficiency and noise tolerance.
\newblock \emph{J. Comput. Syst. Sci.}, 56\penalty0 (2):\penalty0 191--208,
  1998{\natexlab{a}}.

\bibitem[Aslam and Decatur(1998{\natexlab{b}})]{AslamD98b}
Javed~A. Aslam and Scott~E. Decatur.
\newblock General bounds on statistical query learning and {PAC} learning with
  noise via hypothesis boosting.
\newblock \emph{Inf. Comput.}, 141\penalty0 (2):\penalty0 85--118,
  1998{\natexlab{b}}.

\bibitem[Bassily et~al.(2016)Bassily, Nissim, Smith, Steinke, Stemmer, and
  Ullman]{BassilyNSSSU16}
Raef Bassily, Kobbi Nissim, Adam~D. Smith, Thomas Steinke, Uri Stemmer, and
  Jonathan Ullman.
\newblock Algorithmic stability for adaptive data analysis.
\newblock In \emph{Proceedings of the 48th Annual {ACM} {SIGACT} Symposium on
  Theory of Computing, {STOC} 2016, Cambridge, MA, USA, June 18-21, 2016},
  pages 1046--1059, 2016.

\bibitem[Ben{-}David et~al.(1995)Ben{-}David, Itai, and
  Kushilevitz]{Ben-DavidIK95}
Shai Ben{-}David, Alon Itai, and Eyal Kushilevitz.
\newblock Learning by distances.
\newblock \emph{Inf. Comput.}, 117\penalty0 (2):\penalty0 240--250, 1995.

\bibitem[Blum et~al.(1994)Blum, Furst, Jackson, Kearns, Mansour, and
  Rudich]{BlumFJKMR94}
Avrim Blum, Merrick~L. Furst, Jeffrey~C. Jackson, Michael~J. Kearns, Yishay
  Mansour, and Steven Rudich.
\newblock Weakly learning {DNF} and characterizing statistical query learning
  using fourier analysis.
\newblock In \emph{Proceedings of the Twenty-Sixth Annual {ACM} Symposium on
  Theory of Computing, 23-25 May 1994, Montr{\'{e}}al, Qu{\'{e}}bec, Canada},
  pages 253--262, 1994.

\bibitem[Blum et~al.(2003)Blum, Kalai, and Wasserman]{BlumKW03}
Avrim Blum, Adam Kalai, and Hal Wasserman.
\newblock Noise-tolerant learning, the parity problem, and the statistical
  query model.
\newblock \emph{J. {ACM}}, 50\penalty0 (4):\penalty0 506--519, 2003.

\bibitem[Blum et~al.(2005)Blum, Dwork, McSherry, and Nissim]{BlumDMN05}
Avrim Blum, Cynthia Dwork, Frank McSherry, and Kobbi Nissim.
\newblock Practical privacy: the sulq framework.
\newblock In \emph{Proceedings of the Twenty-fourth {ACM}
  {SIGACT-SIGMOD-SIGART} Symposium on Principles of Database Systems, June
  13-15, 2005, Baltimore, Maryland, {USA}}, pages 128--138, 2005.

\bibitem[Bshouty and Feldman(2002)]{BshoutyF02}
Nader~H. Bshouty and Vitaly Feldman.
\newblock On using extended statistical queries to avoid membership queries.
\newblock \emph{Journal of Machine Learning Research}, 2:\penalty0 359--395,
  2002.

\bibitem[Chu et~al.(2006)Chu, Kim, Lin, Yu, Bradski, Ng, and
  Olukotun]{ChuKLYBNO06}
Cheng{-}Tao Chu, Sang~Kyun Kim, Yi{-}An Lin, YuanYuan Yu, Gary~R. Bradski,
  Andrew~Y. Ng, and Kunle Olukotun.
\newblock Map-reduce for machine learning on multicore.
\newblock In \emph{Advances in Neural Information Processing Systems 19,
  Proceedings of the Twentieth Annual Conference on Neural Information
  Processing Systems, Vancouver, British Columbia, Canada, December 4-7, 2006},
  pages 281--288, 2006.

\bibitem[Cortes and Vapnik(1995)]{CortesV95}
Corinna Cortes and Vladimir Vapnik.
\newblock Support-vector networks.
\newblock \emph{Mach. Learn.}, 20\penalty0 (3):\penalty0 273--297, 1995.

\bibitem[Dempster et~al.(1977)Dempster, Laird, and Rubin]{DempsterLR77}
A.~P. Dempster, N.~M. Laird, and D.~B. Rubin.
\newblock Maximum likelihood from incomplete data via the em algorithm.
\newblock \emph{Journal of the Royal Statistical Society, Series B},
  39\penalty0 (1):\penalty0 1--38, 1977.

\bibitem[Diochnos and Tur{\'{a}}n(2009)]{DiochnosT09}
Dimitrios~I. Diochnos and Gy{\"{o}}rgy Tur{\'{a}}n.
\newblock On evolvability: The swapping algorithm, product distributions, and
  covariance.
\newblock In \emph{Stochastic Algorithms: Foundations and Applications, 5th
  International Symposium, {SAGA} 2009, Sapporo, Japan, October 26-28, 2009.
  Proceedings}, pages 74--88, 2009.

\bibitem[Dunagan and Vempala(2008)]{DunaganV08}
John Dunagan and Santosh Vempala.
\newblock A simple polynomial-time rescaling algorithm for solving linear
  programs.
\newblock \emph{Math. Program.}, 114\penalty0 (1):\penalty0 101--114, 2008.

\bibitem[Dwork et~al.(2006)Dwork, McSherry, Nissim, and Smith]{DworkMNS06}
Cynthia Dwork, Frank McSherry, Kobbi Nissim, and Adam~D. Smith.
\newblock Calibrating noise to sensitivity in private data analysis.
\newblock In Shai Halevi and Tal Rabin, editors, \emph{Theory of Cryptography,
  Third Theory of Cryptography Conference, {TCC} 2006, New York, NY, USA, March
  4-7, 2006, Proceedings}, volume 3876 of \emph{Lecture Notes in Computer
  Science}, pages 265--284. Springer, 2006.

\bibitem[Dwork et~al.(2015)Dwork, Feldman, Hardt, Pitassi, Reingold, and
  Roth]{DworkFHPRR15}
Cynthia Dwork, Vitaly Feldman, Moritz Hardt, Toniann Pitassi, Omer Reingold,
  and Aaron Roth.
\newblock The reusable holdout: Preserving validity in adaptive data analysis.
\newblock \emph{Science}, 349\penalty0 (6248):\penalty0 636--638, 2015.
\newblock ISSN 0036-8075.

\bibitem[Feldman(2008)]{Feldman08}
Vitaly Feldman.
\newblock Evolvability from learning algorithms.
\newblock In Cynthia Dwork, editor, \emph{Proceedings of the 40th Annual {ACM}
  Symposium on Theory of Computing, Victoria, British Columbia, Canada, May
  17-20, 2008}, pages 619--628. {ACM}, 2008.

\bibitem[Feldman(2011)]{Feldman11}
Vitaly Feldman.
\newblock Distribution-independent evolvability of linear threshold functions.
\newblock In Sham~M. Kakade and Ulrike von Luxburg, editors, \emph{{COLT} 2011
  - The 24th Annual Conference on Learning Theory, June 9-11, 2011, Budapest,
  Hungary}, volume~19 of \emph{{JMLR} Proceedings}, pages 253--272. JMLR.org,
  2011.

\bibitem[Feldman(2012)]{Feldman12}
Vitaly Feldman.
\newblock A complete characterization of statistical query learning with
  applications to evolvability.
\newblock \emph{J. Comput. Syst. Sci.}, 78\penalty0 (5):\penalty0 1444--1459,
  2012.

\bibitem[Feldman(2017)]{Feldman17}
Vitaly Feldman.
\newblock A general characterization of the statistical query complexity.
\newblock In \emph{Proceedings of the 30th Conference on Learning Theory,
  {COLT} 2017, Amsterdam, The Netherlands, 7-10 July 2017}, pages 785--830,
  2017.

\bibitem[Feldman et~al.(2015)Feldman, Perkins, and Vempala]{FeldmanPV15}
Vitaly Feldman, Will Perkins, and Santosh Vempala.
\newblock On the complexity of random satisfiability problems with planted
  solutions.
\newblock In \emph{Proceedings of the Forty-Seventh Annual {ACM} on Symposium
  on Theory of Computing, {STOC} 2015, Portland, OR, USA, June 14-17, 2015},
  pages 77--86, 2015.

\bibitem[Feldman et~al.(2017)Feldman, Grigorescu, Reyzin, Vempala, and
  Xiao]{FeldmanGRVX17}
Vitaly Feldman, Elena Grigorescu, Lev Reyzin, Santosh~Srinivas Vempala, and
  Ying Xiao.
\newblock Statistical algorithms and a lower bound for detecting planted
  cliques.
\newblock \emph{J. {ACM}}, 64\penalty0 (2):\penalty0 8:1--8:37, 2017.

\bibitem[Fish et~al.(2020)Fish, Reyzin, and Rubinstein]{FishRR20}
Benjamin Fish, Lev Reyzin, and Benjamin I.~P. Rubinstein.
\newblock Sampling without compromising accuracy in adaptive data analysis.
\newblock In Aryeh Kontorovich and Gergely Neu, editors, \emph{Proceedings of
  the 31st International Conference on Algorithmic Learning Theory}, volume 117
  of \emph{Proceedings of Machine Learning Research}, pages 297--318, San
  Diego, California, USA, 08 Feb--11 Feb 2020. PMLR.

\bibitem[Gelfand and Smith(1990)]{GelfandSmith90}
A.~E. Gelfand and A.~F.~M. Smith.
\newblock Sampling based approaches to calculating marginal densities.
\newblock \emph{Journal of the American Statistical Association}, 85:\penalty0
  398--409, 1990.

\bibitem[Grigorescu et~al.(2011)Grigorescu, Reyzin, and
  Vempala]{GrigorescuRV11}
Elena Grigorescu, Lev Reyzin, and Santosh Vempala.
\newblock On noise-tolerant learning of sparse parities and related problems.
\newblock In \emph{Algorithmic Learning Theory - 22nd International Conference,
  {ALT} 2011, Espoo, Finland, October 5-7, 2011. Proceedings}, pages 413--424,
  2011.

\bibitem[Hardt and Ullman(2014)]{HardtU14}
Moritz Hardt and Jonathan Ullman.
\newblock Preventing false discovery in interactive data analysis is hard.
\newblock In \emph{55th {IEEE} Annual Symposium on Foundations of Computer
  Science, {FOCS} 2014, Philadelphia, PA, USA, October 18-21, 2014}, pages
  454--463, 2014.

\bibitem[Janoos et~al.(2018)Janoos, Mohri, and Sridharan]{ALT2018}
Firdaus Janoos, Mehryar Mohri, and Karthik Sridharan, editors.
\newblock \emph{Algorithmic Learning Theory, {ALT} 2018, 7-9 April 2018,
  Lanzarote, Canary Islands, Spain}, volume~83 of \emph{Proceedings of Machine
  Learning Research}, 2018. {PMLR}.

\bibitem[Jerrum(1992)]{Jerrum92}
Mark Jerrum.
\newblock Large cliques elude the metropolis process.
\newblock \emph{Random Struct. Algorithms}, 3\penalty0 (4):\penalty0 347--360,
  1992.

\bibitem[Joseph et~al.(2020)Joseph, Mao, and Roth]{JosephMR20}
Matthew Joseph, Jieming Mao, and Aaron Roth.
\newblock Exponential separations in local differential privacy.
\newblock In Shuchi Chawla, editor, \emph{Proceedings of the 2020 {ACM-SIAM}
  Symposium on Discrete Algorithms, {SODA} 2020, Salt Lake City, UT, USA,
  January 5-8, 2020}, pages 515--527. {SIAM}, 2020.

\bibitem[Kanade(2011)]{Kanade11}
Varun Kanade.
\newblock Evolution with recombination.
\newblock In \emph{{IEEE} 52nd Annual Symposium on Foundations of Computer
  Science, {FOCS} 2011, Palm Springs, CA, USA, October 22-25, 2011}, pages
  837--846, 2011.

\bibitem[Kasiviswanathan et~al.(2011)Kasiviswanathan, Lee, Nissim,
  Raskhodnikova, and Smith]{KasiviswanathanLNRS11}
Shiva~Prasad Kasiviswanathan, Homin~K. Lee, Kobbi Nissim, Sofya Raskhodnikova,
  and Adam~D. Smith.
\newblock What can we learn privately?
\newblock \emph{{SIAM} J. Comput.}, 40\penalty0 (3):\penalty0 793--826, 2011.

\bibitem[Kearns(1998)]{Kearns98}
Michael~J. Kearns.
\newblock Efficient noise-tolerant learning from statistical queries.
\newblock \emph{J. {ACM}}, 45\penalty0 (6):\penalty0 983--1006, 1998.

\bibitem[Kirkpatrick et~al.(1983)Kirkpatrick, Gelatt, and
  Vecchi]{KirkpatrickGV83}
S.~Kirkpatrick, C.~D. Gelatt, and M.~P. Vecchi.
\newblock Optimization by simulated annealing.
\newblock \emph{Science}, 220\penalty0 (4598):\penalty0 671--680, 1983.
\newblock ISSN 0036-8075.

\bibitem[Mitra et~al.(2004)Mitra, Murthy, and Pal]{MitraMP04}
Pabitra Mitra, C.~A. Murthy, and Sankar~K. Pal.
\newblock A probabilistic active support vector learning algorithm.
\newblock \emph{{IEEE} Trans. Pattern Anal. Mach. Intell.}, 26\penalty0
  (3):\penalty0 413--418, 2004.

\bibitem[O'Donnell(2014)]{ODonell14}
Ryan O'Donnell.
\newblock \emph{Analysis of boolean functions}.
\newblock Cambridge University Press, 2014.

\bibitem[Pearson(1900)]{Pearson00}
Karl Pearson.
\newblock On the criterion that a given system of deviations from the probable
  in the case of a correlated system of variables is such that it can be
  reasonably supposed to have arisen from random sampling.
\newblock \emph{The London, Edinburgh, and Dublin Philosophical Magazine and
  Journal of Science}, 50\penalty0 (302):\penalty0 157--175, 1900.

\bibitem[Peikert(2014)]{Peikert14}
Chris Peikert.
\newblock Lattice cryptography for the internet.
\newblock In Michele Mosca, editor, \emph{Post-Quantum Cryptography - 6th
  International Workshop, PQCrypto 2014, Waterloo, ON, Canada, October 1-3,
  2014. Proceedings}, volume 8772 of \emph{Lecture Notes in Computer Science},
  pages 197--219. Springer, 2014.

\bibitem[Robbins and Monro(1951)]{Robbins51}
Herbert Robbins and Sutton Monro.
\newblock A stochastic approximation method.
\newblock \emph{Ann. Math. Statist.}, 22\penalty0 (3):\penalty0 400--407, 09
  1951.

\bibitem[Schapire(1990)]{Schapire90}
Robert~E. Schapire.
\newblock The strength of weak learnability.
\newblock \emph{Machine learning}, 5\penalty0 (2):\penalty0 197--227, 1990.

\bibitem[Sherstov(2008)]{Sherstov08}
Alexander~A. Sherstov.
\newblock Halfspace matrices.
\newblock \emph{Computational Complexity}, 17\penalty0 (2):\penalty0 149--178,
  2008.

\bibitem[Sherstov(2018)]{Sherstov18}
Alexander~A. Sherstov.
\newblock Compressing interactive communication under product distributions.
\newblock \emph{{SIAM} J. Comput.}, 47\penalty0 (2):\penalty0 367--419, 2018.

\bibitem[Simon(2007)]{Simon07}
Hans~Ulrich Simon.
\newblock A characterization of strong learnability in the statistical query
  model.
\newblock In Wolfgang Thomas and Pascal Weil, editors, \emph{{STACS} 2007, 24th
  Annual Symposium on Theoretical Aspects of Computer Science, Aachen, Germany,
  February 22-24, 2007, Proceedings}, volume 4393 of \emph{Lecture Notes in
  Computer Science}, pages 393--404. Springer, 2007.

\bibitem[Steinhardt et~al.(2016)Steinhardt, Valiant, and Wager]{SteinhardtVW16}
Jacob Steinhardt, Gregory Valiant, and Stefan Wager.
\newblock Memory, communication, and statistical queries.
\newblock In \emph{Proceedings of the 29th Conference on Learning Theory,
  {COLT} 2016, New York, USA, June 23-26, 2016}, pages 1490--1516, 2016.

\bibitem[Sz{\"{o}}r{\'{e}}nyi(2009)]{Szorenyi09}
Bal{\'{a}}zs Sz{\"{o}}r{\'{e}}nyi.
\newblock Characterizing statistical query learning: Simplified notions and
  proofs.
\newblock In \emph{Algorithmic Learning Theory, 20th International Conference,
  {ALT} 2009, Porto, Portugal, October 3-5, 2009. Proceedings}, pages 186--200,
  2009.

\bibitem[Tanner and Wong(1987)]{TannerW87}
M~Tanner and W~Wong.
\newblock The calculation of posterior distributions by data augmentation (with
  discussion).
\newblock \emph{Journal of the American Statistical Association}, 82:\penalty0
  528--550, 1987.

\bibitem[Valiant(2015)]{Valiant15}
Gregory Valiant.
\newblock Finding correlations in subquadratic time, with applications to
  learning parities and the closest pair problem.
\newblock \emph{J. {ACM}}, 62\penalty0 (2):\penalty0 13:1--13:45, 2015.

\bibitem[Valiant(1984)]{Valiant84}
Leslie~G. Valiant.
\newblock A theory of the learnable.
\newblock \emph{Commun. {ACM}}, 27\penalty0 (11):\penalty0 1134--1142, 1984.

\bibitem[Valiant(1985)]{Valiant85}
Leslie~G. Valiant.
\newblock Learning disjunction of conjunctions.
\newblock In Aravind~K. Joshi, editor, \emph{Proceedings of the 9th
  International Joint Conference on Artificial Intelligence. Los Angeles, CA,
  USA, August 1985}, pages 560--566. Morgan Kaufmann, 1985.

\bibitem[Valiant(2009)]{Valiant09}
Leslie~G Valiant.
\newblock Evolvability.
\newblock \emph{Journal of the ACM (JACM)}, 56\penalty0 (1):\penalty0 3, 2009.

\bibitem[Vapnik and Chervonenkis(2015)]{VapnikC15}
Vladimir~N Vapnik and A~Ya Chervonenkis.
\newblock On the uniform convergence of relative frequencies of events to their
  probabilities.
\newblock In \emph{Measures of complexity}, pages 11--30. Springer, 2015.

\bibitem[\v{C}ern\'{y}(1985)]{Cerny85}
V.~\v{C}ern\'{y}.
\newblock {Thermodynamical approach to the traveling salesman problem: An
  efficient simulation algorithm}.
\newblock \emph{Journal of Optimization Theory and Applications}, 45\penalty0
  (1):\penalty0 41--51, January 1985.
\newblock ISSN 0022-3239.

\bibitem[Yang(2005)]{Yang05}
Ke~Yang.
\newblock New lower bounds for statistical query learning.
\newblock \emph{J. Comput. Syst. Sci.}, 70\penalty0 (4):\penalty0 485--509,
  2005.

\end{thebibliography}

\end{document}